\documentclass[letterpaper, 10 pt, conference]{ieeeconf}  

\IEEEoverridecommandlockouts                              

\usepackage{times}
\usepackage{cite}
\usepackage{multicol}
\usepackage[bookmarks=true]{hyperref}
\usepackage{amsmath,amssymb,amsfonts}

\usepackage{graphicx}
\usepackage{float}
\usepackage{textcomp}
\usepackage{xcolor}
\usepackage{balance}
\usepackage{lipsum}
\let\labelindent\relax
\usepackage{enumitem}
\usepackage[english]{babel}
\usepackage{algorithm}
\usepackage{algorithmic}
\usepackage{booktabs}

\newtheorem{theorem}{Theorem}

\newtheorem{assumption}{Assumption}

\DeclareMathOperator*{\argmin}{arg\,min}

\newtheorem{definition}{Definition}
\newtheorem{proposition}{Proposition}
\newtheorem{problem}{Problem}

\title{\LARGE \bf
Sample-Based Hybrid Mode Control: Asymptotically Optimal Switching of Algorithmic and Non-Differentiable Control Modes
}

\author{Yilang Liu$^{1}$, Haoxiang You$^{1}$, Ian Abraham$^{1}$
\thanks{$^{1}$Yilang Liu, Haoxiang You, and Ian Abraham are with the Department of Mechanical Engineering and Material Science, Yale University, 17 Hillhouse Avenue, New Haven, CT 06520, USA
        {\tt\small yilang.liu@yale.edu; haoxiang.you@yale.edu; ian.abraham@sydney.edu.au}}}

\begin{document}



\maketitle
\thispagestyle{empty}
\pagestyle{empty}

\begin{abstract}

    This paper investigates a sample-based solution to the hybrid mode control problem across non-differentiable and algorithmic hybrid modes. 
    Our approach reasons about a set of hybrid control modes as an integer-based optimization problem where we select what mode to apply, when to switch to another mode, and the duration for which we are in a given control mode. 
    A sample-based variation is derived to efficiently search the integer domain for optimal solutions.
    We find our formulation yields strong performance guarantees that can be applied to a number of robotics-related tasks. 
    In addition, our approach is able to synthesize complex algorithms and policies to compound behaviors and achieve challenging tasks. 
    Last, we demonstrate the effectiveness of our approach in a real-world robotic examples that requires reactive switching between long-term planning and high-frequency control. Videos are available on  \url{https://yilangliu.github.io/hybrid_mode_sampling/}
    \end{abstract}

\section{INTRODUCTION}
    
    Modern agile robotic systems must dynamically switch between discrete modes—such as making and breaking contacts—to synthesize complex behaviors like locomotion and manipulation. 
    Traditional continuous control methods struggle with these abrupt mode switches, often resulting in instability or suboptimal performance.
    These events often cause abrupt changes in dynamics and constraints, requiring either highly reactive control solutions or more algorithmic planning-based controllers that can handle contact-based reasoning. 
    While it is possible to construct controllers within each of these distinct operating conditions, optimally transitioning between modes becomes challenging, especially when the underlying task requires multiple transitions with multiple hybrid modes. 

    Hybrid control theory offers principled methods that address these challenges by explicitly coordinating discrete mode switches with continuous control inputs. 
    It specifically addresses the fact that many robotic systems undergo mode changes\cite{hybridctrl_bipedal, hybridctrl_bipedal2,seqactionctrl}. 
    For instance, a legged robot system periodically transitions from stance to swing phases \cite{swithtimeop}, while an in-hand manipulation task switches between grasping and free motion \cite{swithtimeop3}. 
    However, adapting hybrid control to arbitrary control modes, e.g., ones that solve algorithmic-based control, or require reasoning of contact dynamics becomes prohibitively challenging due to the combinatorial complexity of switching mode optimization~\cite{hybridctrl4, liu2020design}.

    In this work, we propose a variation to the hybrid control problem that can handle algorithmic and non-differentiable control modes.  
    Our approach formulates the hybrid control problem as a recursive integer-based search problem over the (1) the hybrid mode, (2) the discrete application time of the mode, and (3) the duration of the control mode. 
    Since we operate in discrete-time the underlying search problem can be solved exactly through iterative search. 
    We propose a sample-based version to uniformly sample the exhaustive set without replacement. 
    We find that the sample-based approach has strong convergence guarantees and can find mode switching sequences without needing to reason about the composition of the specific mode. 
    We compare against modern trajectory optimization techniques~\cite{Mppi, predictivesampling, cem} and show that the proposed method is able to achieve significantly better performance by synchronizing more complex hybrid modes. 

    We also validate our approach on a real-world robotic platform, showcasing its practical effectiveness for demanding tasks that transition between stabilizing controllers through bridging of a model-predictive controller. Our contributions can be summarized as below:

    \begin{enumerate}
    \item A novel, iterative sample-based formulation of the hybrid control sequencing problem,
    \item provable performance guarantees on optimizing mode sequencing, and
    \item demonstration of complex mode switching between stabilizing controllers and mpc-based controllers in the real-world quadruped experiment.
\end{enumerate}

   \begin{figure}[t]
      \centering
        \includegraphics[width=\linewidth]{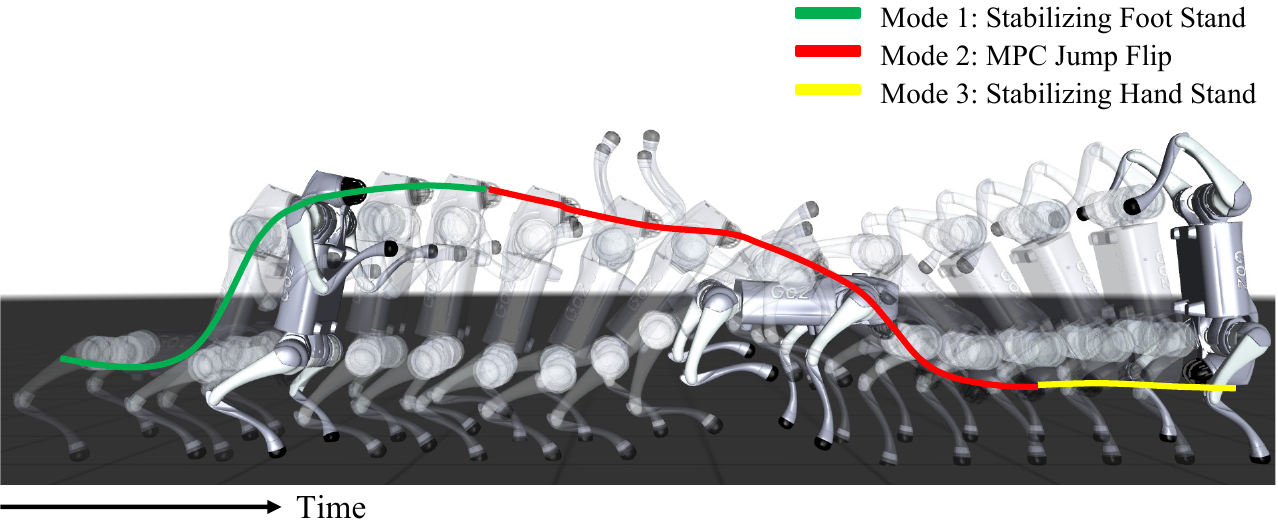}
        \vskip -5pt
        \caption{\textbf{Sample-Based Hybrid Mode Switching Enables Extreme Control Variations.}  We show the time-elapsed figure of the Unitree Go2 quadruped transitioning across three distinctive control modes (foot stand, jump flipping, and handstand) demonstrating that our method can achieve extremely agile motor skills within sharp transition times. The green, red, and yellow lines denote the head motion during each mode transitions. }
        \label{fig:figure_1}
        \vskip -11pt
    \end{figure}

The remainder of the paper is organized as follows: Section \ref{sec:related work} provides related work, Section \ref{sec:preliminaries} introduces the hybrid control problem. Section \ref{sec:method} describes our proposed sample-based hybrid mode control approach, Section \ref{sec:experiments} provides performance results of our proposed method, and last conclusion and limitation are presented in Section \ref{sec:conclusion}.

\section{Related Work}
\label{sec:related work}
\subsection{Hybrid Control}
    The study of hybrid control systems addresses a canonical problem in robotic systems where dynamic switching between discrete modes, such as making or breaking contacts in locomotion or transitioning between manipulation phases \cite{you2025acceleratingvisualpolicylearningparallel, hybridctrl_bipedal2, hybridctrl_bipedal}, needs to be tightly scheduled. Hybrid control problems are typically solved by optimizing with switched dynamical systems \cite{1272934}, time \cite{6669577}, or controllers \cite{EGERSTEDT1999617}. However, they struggle to scale to high-dimensional systems due to two key challenges: (1) the objective landscape is highly nonconvex \cite{seqactionctrl}, and (2) computational burden becomes intractable due to an increasing number of mode switches \cite{6426270}.

     To mitigate these issues, prior works consider designing multiple simplified dynamics systems for high-dimensional hybrid control tasks \cite{simp_model2}. For example, the quadruped dynamical systems are separated into front leg control and back leg control to ensure stable foot placements for the robot \cite{6094601}. However, using a simplified dynamical system loses the possibility to fully exploit the whole-body capability. Other methods predefine the mode sequence in which contacts are made and broken. In this case, the optimization problem becomes easy to solve \cite{simp_model5}. However, predefined modes constrain the robot's behavior to a predetermined trajectory.

     In this work, we proposed a sampling-based approach for solving the hybrid control problem. Our method does not rely on pre-sequencing modes or simplified models and can synthesize complex and agile behaviors. Moreover, our method synthesizes extreme motor skills and conducts computations online.

\subsection{Sample-based Control}
    Sample-based control methods \cite{Mppi2, es, christian_2} have recently emerged as a simple yet effective approach for solving high-dimensional robotics tasks such as legged locomotion  \cite{fullordersampling, You_1, liu_1} and manipulation \cite{sample_manipulation1,sample_manipulation2,hybridcontrol2, sample_manipulation3, wang_3}. Rather than relying on explicit gradient-based techniques, sample-based controls begin by sampling a control sequence from an initial distribution, executing a forward rollout of that sequence, and then adjusting the sampling distribution based on the resulting costs. This gradient-free nature makes sample-based control particularly well-suited for complex, non-differentiable systems where traditional optimization methods may struggle. Recent studies also show that stronger derivative modeling and stable value-learning formulations can directly improve control optimization behavior \cite{wang_1, wang_2, You_2}.

    Despite the effectiveness, sample-based control methods exhibit two limitations. First, they often treat control at each timestep as an independent variable, ignoring the inherent hybrid structure of robotic tasks \cite{liu2010sampling}. Second, the number of samples required to adequately explore the control space scales exponentially with the planning horizon. This exponential growth makes sample-based methods computationally infeasible for long-horizon tasks.

    \emph{Our work unifies the strengths of hybrid control and sample-based optimization}. By reparameterizing the key decision variables, i.e., \textit{when} and \textit{how long} to apply control, our approach can discover complex behaviors in contact-rich scenarios Additionally, the number of decision variables in our formulation is independent of time horizon, effectively reducing the search space for long-horizon tasks like in-hand manipulation. As a result, we effectively mitigate the exponential growth in sample requirements, making it feasible to handle both high-dimensional and long-horizon robotic tasks.
    
\section{Problem Formulation: The (Indefinite) Hybrid Control Problem}
\label{sec:preliminaries}
    In this section, we present a general formulation of the \emph{indefinite} hybrid-mode switching control problem. 
    \begin{definition} \textit{(Hybrid Mode System)}
        Let the hybrid (autonomous) dynamical system with $M \in \mathbb{Z}_+$ modes be defined as 
        \begin{equation}\label{eq:dynamics}
            \dot{x} = f_m(x(t), t), \, m  \in \mathcal{M}
        \end{equation}
        where $\mathcal{M}=\{1,2,\ldots, M \}$, $x(t) \in \mathcal{X} \subseteq \mathbb{R}^n$ is the state, and $f_m: \mathcal{X} \times \mathbb{R}_+ \to \mathcal{T}_\mathcal{X}$ are (potentially non-differentiable or algorithmic) hybrid modes. 
    \end{definition}
    An example of non-differentiable and algorithmic hybrid modes are state transitions subject to open-loop model-predictive controller (e.g., model-predictive path integral control~\cite{Mppi}), learned policies~\cite{learningwalk}, or high-frequency whole-body controllers~\cite{wholebodympc}.
    
    We interested in minimizing a performance metric $\mathcal{J} : \mathcal{X}\times[0,t_f] \to \mathbb{R}$ for some time window $t \in [0,t_f]$ by selecting a sequence of modes $\{m_1, \dots m_i, \ldots m_I\}$ for $I<\infty$ instances that are applied at non-overlapping time $\tau_i \in [0,t_f]$ for a time duration of $\lambda_i \le t_f - \tau_i$. 
    \begin{definition} \textit{(Continuous-Time Hybrid Mode Switching Set)}
        The set of non-overlapping switches is given as  
        \begin{equation}
            \begin{split}
                \mathcal{T} = &\Big\{ (m, \tau, \lambda)_i \,\, | \,\,\forall i \in \{1, \ldots , I\}, \\ 
                & \tau_i \in [0,t_f], \lambda_i \le t_f - \tau_i, \tau_i+\lambda_i=\tau_{i+1} \Big\}
            \end{split} \in \mathbb{M}
        \end{equation}
        where the tuple $(m,\tau,\lambda)_i$ represent the $i^\text{th}$ mode transition at time $\tau$ for a duration of $\lambda$, and $\mathbb{M}$ is the set of all possible mode transitions. 
    \end{definition}

        Given the definition for the set of non-overlapping switches and hybrid modes, we formalize the optimization problem to find the optimal switching set $\mathcal{T}^\star$.
        
        \begin{problem} \textit{(The Indefinite Hybrid Mode Switching Problem)}
        Let $\mathcal{J} : \mathbb{M} \to \mathbb{R}$ be a performance metric and $\mathcal{M}$ be a set of valid hybrid modes. Then the set of mode transitions $\mathcal{T} \in \mathbb{M}$ that minimizing $\mathcal{J}$ is given as 
        \begin{align}\label{eq:indefinite_hcp}
            \begin{split}
            &\min_\mathcal{T} \, \mathcal{J}(\mathcal{T}) = \int_{t=0}^{t_f} \ell(x(s)) ds \\ 
            & \text{subject to} \\ 
            & \mathcal{T} = \Big\{ (m, \tau, \lambda)_i \,\, | \,\,\forall i \in \{1, \ldots , I\}, \\ 
            & \hspace{5em} \tau_i \in [0,t_f], \lambda_i \le t_f - \tau_i, \tau_i+\lambda_i=\tau_{i+1} \Big\}\\ 
            & \dot{x} = f_m(x(t), t), x(0) = x_0, m_i \in \mathcal{M}
            \end{split}
        \end{align}
        where $\ell: \mathcal{X} \to \mathbb{R}$ is a cost function that defines a high-level objective.
        \end{problem}

        In general~\eqref{eq:indefinite_hcp} is challenging to solve as (1) $f_m$ needs to be at least continuous and differentiable\footnote{~\cite{ansari_sequential_nodate} presents a solutions for impacting systems, but does require gradient knowledge before and after the impact.}, and (2) there can be an \emph{indefinite} number of switches $I$ which makes the problem intractable. 
        Existing hybrid control methods generally focus on solving the indefinite switching through formulating variations of~\eqref{eq:indefinite_hcp} by either defining a finite set of mode transitions (i.e., $I$ is given)~\cite{QuadrupedCapturability}, or iteratively solving for single mode transitions one at a time and ``stitching'' together the switching mode set. 
        More specifically, we define the single-mode hybrid control problem as a means to reconcile infinitely many switches into a single mode switching problem. 
        \begin{problem} \textit{(The Single Switch Hybrid Mode Control Problem)}
            Let $\mathcal{J} : \mathbb{M} \to \mathbb{R}$ be a performance metric and $\mathcal{T} \in \mathbb{M}$ be a nominal switching sequence. Then the set of single mode transition tuple $(m, \tau, \lambda)$ that minimizing $\mathcal{J}$ is given as 
        \begin{align}\label{eq:single_hcp}
            \begin{split}
            &\min_{(m, \tau, \lambda)} \, \mathcal{J}( (m, \tau, \lambda) \cup \mathcal{T}) = \int_{t=0}^{t_f} \ell(x(s)) ds \\ 
            & \text{subject to} \\ 
            & \mathcal{T} = \Big\{ (m, \tau, \lambda)_i \,\, | \,\,\forall i \in \{1, \ldots , I\}, \\ 
            & \hspace{5em} \tau_i \in [0,t_f], \lambda_i \le t_f - \tau_i, \tau_i+\lambda_i=\tau_{i+1} \Big\}\\ 
            & \dot{x} = f_m(x(t), t), x(0) = x_0, m_i \in \mathcal{M}
            \end{split}
        \end{align}
        where the operation $(m, \tau, \lambda) \cup \mathcal{T}$ stitches the single mode transition tuple $(m, \tau, \lambda)$ with a given default switching mode set $\mathcal{T}$. 
        \end{problem}

        Iteratively solving~\eqref{eq:single_hcp} can eventually lead to solving~\eqref{eq:indefinite_hcp}~\cite{ansari_sequential_nodate}; however, existing solutions are limited by the choice of hybrid mode which leads to performance gaps, minimal solution guarantees, and often requires heuristic-based mode switching~\cite{ansari_sequential_nodate}.
        Integrating and scheduling non-differentiable, and algorithmic hybrid modes has the potential to utilize more complex behaviors and reduce the performance gap, but scheduling remains an open challenge.
        In the following section, we present a potential solution to this problem via sample-based optimization.

    \section{Sample-Based Hybrid Mode Control} \label{sec:method}
    
        We propose to solve~\eqref{eq:indefinite_hcp} by first formulating a discrete-time variation of the indefinite hybrid mode-switching problem. 
        Then, we show that one can iteratively solve for singular switching modes via sample-based optimization that yield strong asymptotic convergence guarantees. 
        
        \subsection{The Discrete (Definite) Hybrid Control Problem}

            Let us define a discrete time index $k$ where $t_k = k \Delta t$ is the time at the $k^\text{th}$ index, and $\Delta t \in\mathbb{R}_+$ is a time discretization. 
            Here, we choose $\Delta t$ to be the slowest hybrid state mode transition $f_m$ and
            \begin{equation}
            \begin{split}
                x_{k+1} &= F_m(x_k, k) \\ 
                &= x(t_k + \Delta t) = x(t_k) + \int_{s={t_k}}^{t_k + \Delta t} f_m(x(s), s) ds
            \end{split}
            \end{equation}
            is the mode transition function. In general, $f_m$ may not be continuous and $F_m$ is solved via an algorithmic optimization (e.g., contact-based dynamics~\cite{hybridctrl_1}). 
            The idea is that the underlying control problem with state feedback will operate on some \emph{fixed} frequency (i.e., discretized time).  
            Thus, we abandon continuous-time formulation (which is limited in runtime computation) for a simpler formulation of mode scheduling in discrete-time. Next, we define the set of mode transition in discrete-time, 
            \begin{definition} \textit{(Discrete-Time Hybrid Mode Switching Set)}
                Let $T \Delta t = t_f$ be the discrete planning time and $\mu_i, \nu_i \in \mathbb{Z}_+$ are the discrete-time mode application and duration times respectively. 
                Then the discrete-time hybrid mode switching set is given as 
            \begin{align}
                \begin{split}
                    \mathcal{K} = &\Big\{ (m, \mu, \nu)_i \,\, | \,\,\forall i \in \{1, \ldots , I\}, \\ 
                    & \hspace{1em} \mu_i \in [0,T], \nu_i \in [0,T - \mu_i], \mu_i+\nu_i=\mu_{i+1} \Big\}
                \end{split} \in \mathcal{M}^T
            \end{align}             
            \end{definition}
            where $\mathcal{M}^T \subset \mathbb{Z}_+^T$ is the space of $T$ hybrid mode sequences. \footnote{The set here in discrete-time reduces to a sequence of modes $m_i$ that starts at discrete time $\mu_i$ repeat according to $\nu_i$ or a vector of integer values, e.g., $\mathcal{K}=\{1,4,4,4,5,2,1,1\}$ for $\mathcal{M}=[1,\ldots, 5]$. } 

            Discretizing the performance metric $\mathcal{J}$ leads to the following mode-scheduling optimization problem. 
            \begin{problem} \textit{(The Discrete-Time Hybrid Mode Switching Problem)}
                Let $\mathcal{K} \in \mathcal{M}^T$ and $\mathcal{J} : \mathcal{M}^T \to \mathbb{R}$ be a discretized performance metric with $T \Delta t = t_f$ as the discretized planning time. 
                Then an optimal mode sequence $\mathcal{K}^\star$ can be obtained by solving the following problem
                \begin{equation}\label{eq:indefinite_discrete_hcp}
                    \begin{split}
                    &\min_\mathcal{K} \, \mathcal{J}(\mathcal{K}) = \sum_{k=0}^{T} \ell(x_k) \Delta t \\ 
                    & \text{subject to} \\ 
                    & \mathcal{K} = \Big\{ (m, \mu, \nu)_i \,\, | \,\,\forall i \in \{1, \ldots , I\}, \\ 
                    & \hspace{2em} \mu_i \in [0,T], \nu_i \in [0,T - \mu_i], \mu_i+\nu_i=\mu_{i+1} \Big\} \\
                    & x_{k+1} = F_m(x_k, k), x(0) = x_0, m_i \in \mathcal{M}
                    \end{split}.
                \end{equation}                
            \end{problem}
            Note that there are at most $T$ hybrid mode transitions which makes this problem definite with $\mathcal{K}=\{m_0, m_1, \ldots, m_{T-1}\} \in \mathcal{M}^T$ being a vector containing $T$ transition modes (that may repeat according to the discrete duration).
            In fact, since $m, \mu$ and $\nu$ are all \emph{discrete integers}, the problem can be \emph{exactly} solved through brute-force search that scales $\mathcal{O}(M^T)$.
            Exact solutions to~\eqref{eq:indefinite_discrete_hcp} can be obtained with modern GPUs; however, a simpler recursive formulation can be derived which is used to inform us of formal performance guarantees. 

        \subsection{Iterative Hybrid Mode Control Sequencing}
            In order to solve~\eqref{eq:indefinite_discrete_hcp}, we propose a iterative method rather than a brute-force approach that provides formal guarantees. 
            \begin{problem}\textit{(Discrete-Time Single Switch Hybrid Mode Control Problem)}
                Let us define a discrete default mode transition set $\mathcal{K}_\text{def} \in \mathcal{M}^T$. 
                Then, the solution $(m, \mu, \nu)^\star$ that minimizes $\mathcal{J}$ is given as 
                \begin{equation}\label{eq:single_discrete_hcp}
                    \begin{split}
                    &(m, \mu, \nu)^\star = \argmin_{(m, \mu, \nu)} \, \mathcal{J}((m, \mu, \nu) \cup\mathcal{K}_\text{def}) = \sum_{k=0}^{T} \ell(x_k) \Delta t \\ 
                    & \text{subject to} \\ 
                    & \mathcal{K}_\text{def} = \Big\{ m_i \forall i \in [0,T-1]\Big\}\\ 
                    & x_{k+1} = F_m(x_k, k), x(0) = x_0, m_i \in \mathcal{M}
                    \end{split}
                \end{equation}
                where $(m, \mu, \eta) \cup\mathcal{K}_\text{def}$ produces an updated default schedule $\mathcal{K}_\text{def}(k) = m \forall k\in[\mu, \mu+\eta]$.
            \end{problem}

            Problem~\eqref{eq:indefinite_discrete_hcp} can be solved through iterative refinement of $\mathcal{K}_\text{def}$ to obtain the optimal discrete-time mode transition set $\mathcal{K}^\star$. Here,~\eqref{eq:single_discrete_hcp} can be exactly solved via brute-force search with complexity $\mathcal{O}(MT(T+1)/2)$ which is much faster than $\mathcal{O}(M^T)$. We outline the iterative approach in Algorithm~\ref{alg:1}. 
            
            \begin{assumption}
                There exists a sequence of modes modes $\mathcal{K}^\star =\{ m_i \}_{i=1}^M$ such that $\mathcal{J}(\mathcal{K}^\star)$ is minimized. 
            \end{assumption}
            
            \begin{proposition} \label{prop:1} \textit{(Locally Optimal Schedules)} Let $f_m \forall m \in \mathcal{M}$ be deterministic, $\mathcal{K}^\star$ the optimal switching mode schedule, and $\mathcal{J}(\mathcal{K}^\star)$ a local optima. Then, $\mathcal{K}$ is optimal if $\forall \mu \in[0, T-1]$ and $\forall m \in \mathcal{M}$, $\nexists \nu \in [1, T-\mu]$ such that 
            \begin{equation}
                \mathcal{J}((m, \mu, \nu)\cup \mathcal{K}) \le \mathcal{J}(\mathcal{K}),
            \end{equation} 
            $\mathcal{J}(\mathcal{K}) = \mathcal{J}(\mathcal{K}^\star)$, and $\mathcal{K}=\mathcal{K}^\star$ is a local optima. 
            \end{proposition}
            \begin{proof}
                Enumerating through each mode $m$ and discrete application time $\mu$ and checking for all $\nu \in [1,T-\mu]$, if the cost does not reduce, then the only viable solution is to apply all modes for a duration $\nu=0$ resulting in no reduction of the cost, and the schedule $\mathcal{K}$ is at a local optima. 
            \end{proof}

            Now that we know of the existence of local optima and a condition for which to check. 
            We seek to prove that iteratively solving~\eqref{eq:single_discrete_hcp} leads to the local optimal solution. 

            \begin{theorem} \label{thm:1}\textit{(Asymptotic Convergence)} Let $\mathcal{J} : \mathcal{M}^T \to \mathbb{R}$, and $V_{(m, \mu, \eta)^\star}:\mathcal{M}^T\to\mathcal{M}^T$  where $V_{(m, \mu, \eta)^\star}(\mathcal{K}^k) = (m,\mu, \nu)^\star \cup \mathcal{K}^k = \mathcal{K}^{k+1}$ is an update equation to the $k^\text{th}$ hybrid transition set with solution~\eqref{eq:single_discrete_hcp}.
            Then,
            \begin{align}
                \mathcal{J}(V_{(m, \mu, \eta)^\star}(\mathcal{K})) \le J(\mathcal{K})
            \end{align}
            and $\exists K \in\mathbb{Z}_+$ such that 
            \begin{equation}
                V_{(m, \mu, \eta)^\star}(\mathcal{K}^{K+1}) = V_{(m, \mu, \eta)^\star}(\mathcal{K}^K)
            \end{equation}
            is a fixed-point for $k \in [0,\ldots, K]$.
            \end{theorem}
            \begin{proof}
                Using Proposition~\ref{prop:1}, if $\exists m,\mu$ where $\nu \neq 0$ in~\eqref{eq:single_discrete_hcp}, then $\mathcal{J}(V_{(m, \mu, \eta)^\star}(\mathcal{K}^{k+1})) < J(\mathcal{K}^k)$. If $\nu=0 \forall m,\mu$ is the minimizer to~\eqref{eq:single_discrete_hcp}, then $(m,\mu, 0) \cup \mathcal{K} = \mathcal{K}$ and $V_{(m, \mu, \eta)^\star}(\mathcal{K}^{k+1}) = V_{(m, \mu, \eta)^\star}(\mathcal{K}^k)$ is a fixed point. 
            \end{proof}

            Because we discretized the underlying mode sequencing problem, Theorem~\ref{thm:1} becomes possible to state as a performance guarantee. 
            The main challenge is efficiently iterating through the solution space of~\eqref{eq:single_discrete_hcp} to exactly find the minimum $(m, \mu, \nu)^\star$. 
            With a GPU, it is possible to iterate through the solution set; however, we present a more efficient sample-based approach that allows to iterate through the solution set without a GPU.  

            \begin{algorithm}
        
            \caption{Iterative Discrete Hybrid Mode Sequencing}
            \label{alg:1}
            \begin{algorithmic}[1]
                \STATE \textbf{Initialize:} control modes $F_m$ for $m \in \mathcal{M}$, default mode sequence $\mathcal{K}_\text{def} = \{m_i \forall i \in [0,T-1] \}$, with $m_i$ is any arbitrary mode, initial condition $x_0$, planning time $T$, cost $\ell$, max iterations $\text{iter}_\text{max}$, 
                \STATE $\text{iter} \gets 0$
                \WHILE{\(\text{iter} < \text{iter}_\text{max}\) or $\mathcal{J}((m, \mu, \nu) \cup\mathcal{K}_\text{def}) \le \mathcal{J}(\mathcal{K}_\text{def})$}
                    \STATE $(m, \mu, \nu)\gets$ Solve Problem~\ref{eq:single_discrete_hcp}
                    \IF{$\nu=0$}
                        \STATE Terminate while loop
                    \ENDIF
                    \STATE $\mathcal{K}_\text{def}(k) = m \forall k \in[\mu, \mu+\nu]$
                    \STATE $\text{iter} \gets \text{iter}+1$
                \ENDWHILE
            \end{algorithmic}
        \end{algorithm}

    \subsection{Sample-Based Hybrid Mode Control}

        Here, we present a variation of the solution to Problem~\eqref{eq:indefinite_discrete_hcp} via a sample-based iteration of Problem~\eqref{eq:single_discrete_hcp}. 
        The rationale for a sampling method is that the search space can be significantly reduced from $\mathcal{O}(M\cdot T\cdot(T+1))$ to only requiring $\mathcal{O}(N)$ samples for $N$ is the number of uniform samples to check a solution. 
        \begin{assumption} \label{ass:1} \textit{(Uniqueness of Optimal Mode Transitions)}
            Given a default mode transition set $\mathcal{K}$, there exists a unique single mode transition $(m, \mu, \nu)$ that minimizes $\mathcal{J}((m, \mu, \nu)\cup \mathcal{K})$.
        \end{assumption}
        For real-valued states $x(t)$ with algorithmic modes $F_m$, Assumption~\ref{ass:1} is valid so long as $\nu \neq 0$, in which case the solver is at a local optima.
        \begin{theorem} \textit{(Convergence of Sample-Based Single-Mode Hybrid Mode Control)}
            Let $\Omega = \{ (m,\mu, \eta) \forall m \in \mathcal{M}, \mu \in [0,T-1], \nu \in[0,T-\mu] \}$ be the set of all possible $Z = M \cdot T \cdot (T+1)/2$ single mode transitions.
            In addition, let $\mathcal{U}_N(\Omega)$ be a uniform distribution that draws $N<Z$ samples without replacement from the set $\Omega$.  
            Then the optimal mode transition tuple $(m, \mu, \nu)^\star$ to Problem~\eqref{eq:single_discrete_hcp} is found with probability $\mathcal{P}((m, \mu, \nu)^\star)=N/Z$ with at most $\frac{Z}{N}$ draws.
        \end{theorem}
        \begin{proof}
            The probability of not finding the optimal single-mode transition is given by $\mathcal{P}(\neg(m, \mu, \nu)^\star) = \prod_{i=1}^N \frac{Z-i}{Z-i+1} = \frac{Z-N}{Z}$. 
            Thus, $\mathcal{P}((m, \mu, \nu)^\star) = 1-\frac{Z-N}{Z} = \frac{N}{Z}$ of drawing the optimal mode transition with at most $\frac{Z}{N}$ draws until $\mathcal{P}((m, \mu, \nu)^\star)$ = 1.
        \end{proof}
        In essence, the proposed sample-based approach takes advantage of the integer-based optimization that amounts to evaluating a list of tuples (the mode transition set). 
        Indeed, there exists many ways to approach this problem that can be done by evaluating the list of mode transitions. 
        Our proposed approach was motivated by the need to only evaluate a subsample of the mode transition set as needed. 
        An outline of the sample-based single-mode hybrid control approach is provided in Algorithm~\ref{alg:2}.
        
\begin{algorithm}
    \caption{Sample-Based Single-Mode Hybrid Mode Control}
    \label{alg:2}
    \begin{algorithmic}[1]
        \STATE \textbf{Initialize:} nominal mode sequence $\mathcal{K}$, sample set $\Omega$, $N$ total number of samples, control modes $F_m$ for $m \in \mathcal{M}$, objective $\mathcal{J}$, planning time $T$, initial condition $x_0$, max iterations $\text{iter}_\text{max} = Z/N$, $\mathcal{J}_\text{best} = \mathcal{J}(\mathcal{K}_\text{def})$.
        \STATE $\text{iter}\gets0$
        \WHILE{\(\text{iter} < \text{iter}_\text{max}\)}
            \STATE $\{ (m, \mu, \nu) \}_{i=1}^N \sim \mathcal{U}_N(\Omega)$
            \FOR{\(i=1,\,\dots,\,N\)}
                \STATE $\mathcal{K}_\text{eval} = \mathcal{K}_\text{def} \cup (m, \mu, \nu)$
                \IF{$\mathcal{J}(\mathcal{K}_\text{eval}) \le \mathcal{J}_\text{best}$} 
                    \STATE \textbf{return} $(m, \mu, \nu)$ 
                \ENDIF
            \ENDFOR
            \STATE $\Omega = \Omega \setminus \{ (m, \mu, \nu) \}_{i=1}^N $
            \STATE $\text{iter} \gets \text{iter} + 1$
        \ENDWHILE
    \end{algorithmic}
\end{algorithm}

    Algorithm~\ref{alg:2} effectively solves Problem~\eqref{eq:indefinite_discrete_hcp} by drawing $N$ uniform samples from the set of single mode transitions $\Omega$ and evaluating the performance of the mode sequence stitched with $\mathcal{K}_\text{def}$. 
    If a mode transition reduces the cost from the default, then we return the mode transition. Otherwise, we remove the drawn samples from the set $\Omega$ and resample. 
    There is a scenario where we return a sampled mode $(m, \mu, \nu)$ that reduces the cost, but another mode transition $(m, \mu, \nu')$ for $\nu' \in [0, T-\mu]$ can further reduce the cost. 
    In this case, we rely on Algorithm~\ref{alg:1} to reinitialize Algorithm~\ref{alg:2} to randomly search the set $\Omega$ once more to the improve the cost until $\nu=0$ is the only valid option (i.e., the only mode transition that maintains the cost as stationary is applying no additional mode transition).
    
    Through this sample-based approach, we are able to more efficiently search for hybrid mode sequences. As a result, we can optimize mode-sequences of complex behaviors on systems that experience non-differentiable contact, require mixing of algorithmic modes, and require composition of low-level behaviors to achieve far greater global behaviors than a single mode could. We illustrate our results in the following Section~\ref{sec:experiments}.

    \section{Experiments and Results}\label{sec:experiments}

    We present both simulation and real-world results for our method, along with comparisons to other sampling-based approaches.

    \textbf{Toy Example: Cartpole Swing Up }
    Here, we consider a cartpole swing-up task. We set the total horizon $T=100$. The state is defined as pole angle \(\theta\), cart position \(p\), pole angular velocity \(\dot{\theta}\), and cart linear velocity \(\dot{p}\). The control is the one-dimensional force applied horizontally to the cart. The pole is initialized in up-down \(\theta=\frac{1}{2}\pi\), and the goal is to swing the pole up-right and maintain balance. Hence, the stage cost is defined as \(c(\mathbf{x}_t,\mathbf{u}_t) = 4.0(\cos(\theta)-1)^2+0.1p^2+0.1(\dot{\theta}^2+\dot{p}^2)+\mathbf{u}^2\) and the terminal cost is defined as \(c_f(\mathbf{x}_T) = 4.0(\cos(\theta)-1)^2\). We choose the simple hybrid mode $\mathcal{M}$ as a set of the discrete controls bounded by the saturation limits. Specifically, controls are uniformly spaced between $u_{min}$ and $u_{max}$

    Figure~\ref{fig: cost for cartpole} demonstrates the results for our method and other sampling approaches. 
    We observe that our approach consistently finds optimal solutions across different horizons by actively considering time as a decision variable.
    In contrast, other sampling-based methods fail to identify good optima as the horizon increases, owing to the rapidly expanding search space. Therefore the performance of other methods further deteriorates with longer horizons, as their limited sample size prevents them from effectively solving the control problem. 
    By comparison, our method continues to improve performance as the horizon increases.

\begin{figure}
    \vskip 7pt
    \centering
    \includegraphics[width=\linewidth]{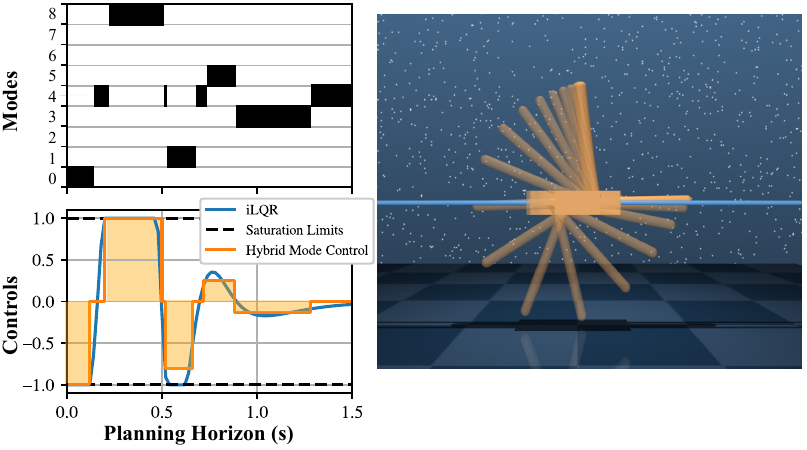}
    \vskip -5pt
    \caption{\textbf{Top Left}: The optimal control mode sequences for the cartpole system, where the Y axis denotes the index for each corresponding sampled modes. \textbf{Bottom Left}: The resulting control applied to the actuators for the cartpole system are shown. The blue curve represents the control generated by the iLQR algorithm. The orange curve represents the executed control found using our sample-based hybrid scheduler, where the shaded areas denote the resulting distinguish modes. \textbf{Right}: The visualization of the cartpole trajectory with controls found by our method.}
    \label{fig:comparison}
    \vskip -10pt
\end{figure}

We also compare the control sequence found by our method to the gradient-based iLQR method in Figure~\ref{fig:comparison}. The control sequence found by our method closely resembles the optimal iLQR sequence, even with a few modes.

\begin{figure}
    \vskip 5pt
    \centering
    \includegraphics[width=\linewidth]{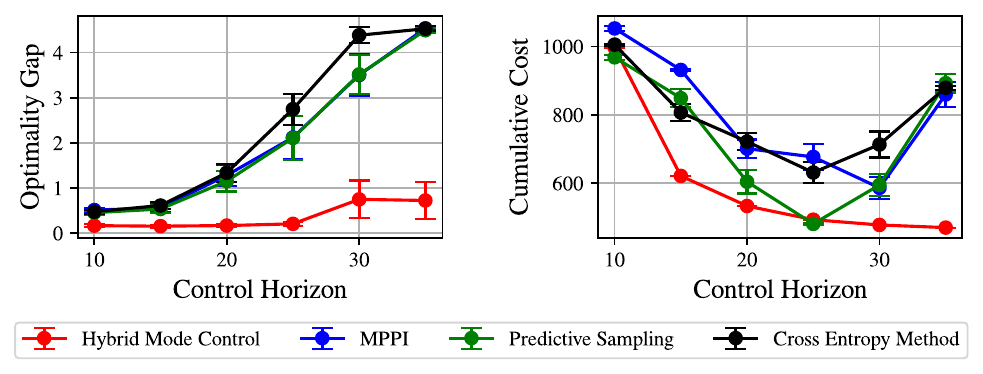}
    \vskip -5pt
    \caption{\textbf{Top}: Results for the short-horizon subproblem of cartpole system. The Y-axis quantifies the optimality gap, defined as the difference between the objective value and a locally optimal solution obtained via iterative LQR. To ensure scale invariance across tasks, we normalize this gap by the planning horizon length $H$. Here error-bar denotes standard deviation across $5$ random seeds. The performance of classical sampling methods deteriorates as the horizon length increases, whereas our method consistently finds good minima, even with a limited sample size. \textbf{Bottom}: Overall performance under a predictive control framework. Both methods show a decrease in cumulative cost as the horizon increases, as a longer horizon reduces myopic decision-making. However, as the horizon continues to grow, classical methods experience an increase in total cost due to their inability to effectively optimize the sub-problem, whereas our method continues to decrease the overall cost.}
    \label{fig: cost for cartpole}
    \vskip -10pt
\end{figure}

\begin{figure*}
    \centering
    \includegraphics[width=\linewidth]{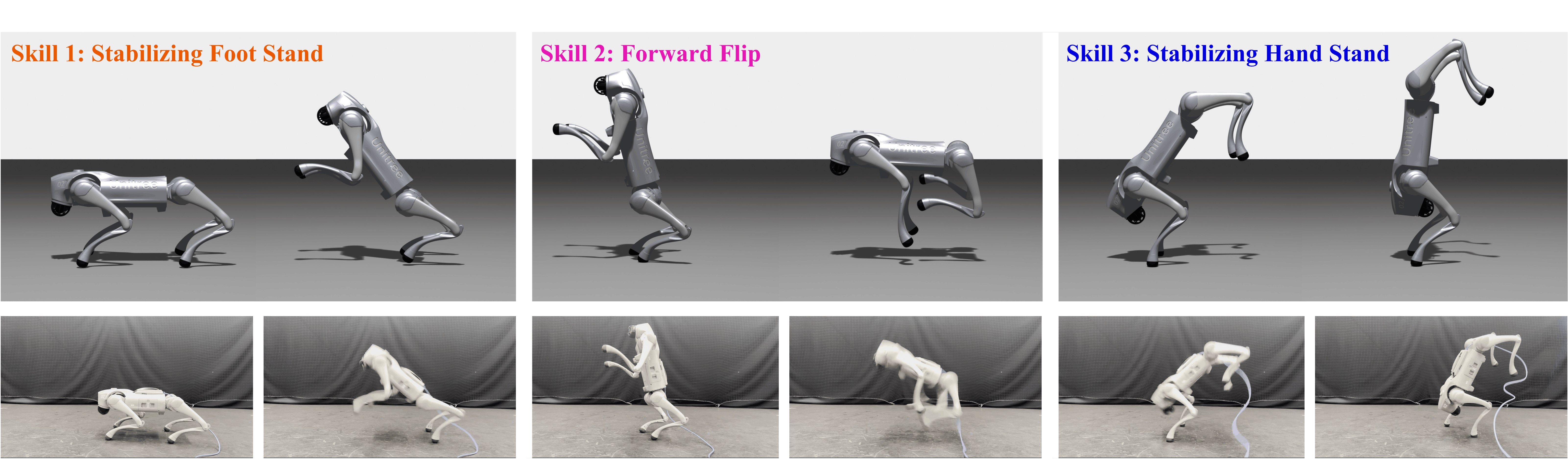}
    \vskip -5pt
    \caption{We verify our method on the real Unitree Go2 quadruped. From left to right, the robot successfully transitions through foot-stand, jump-flip, and hand-stand behaviors. We update the control sequence in real time with onboard sensing and computation, demonstrating the efficiency of our method.}
    \vskip -2pt
    \label{fig:hardware_exps}
\end{figure*}

\subsection{Scalability to High-dimensional Task} Next, we investigate the scalability of our proposed approach to more complex and higher-dimensional control tasks. We start with simulation experiments to demonstrate scalability to high-dimensional tasks shown in Figure \ref{fig:figure_1}. The robot is tasked to stand on its hind legs and perform a flip directly to foreleg stance and balance. The task can be performed in three phases of locomotion, each of which requires distinct control strategies: stabilizing hind leg control policy, high speed predictive forward momentum swing towards the foreleg stance, and a stabilizing foreleg balancing control policy. We specifically choose this experiment because those three stages requires distinct control strategies and whole-body coordination. 

\textbf{Task Design:} The quadruped example uses the Unitree Go2 robot for both simulation and hardware experiment. The dimension for quadruped state space $x_k\in \mathbb{R}^{48}$ and positional control inputs $u_k\in \mathbb{R}^{12}$. Here, state $x_k$ contains the gravity projected in the body frame, base position, base linear and angular velocity, joint positions and velocities, and previous action. Control inputs $u_k$ contain desired robot motor target angles in radians. 

\textbf{Hybrid Mode Design:} We design three hybrid modes consisting of learning-based foot stand and hand stand controller, and a sampling-based model predictive controller (MPC). More modes can be added as long as they generate distinct motor skills. All three controllers share the same observation and action space. Specifically, our hybrid modes are defined as the following: 
\begin{equation}
    \mathcal{M} = \{F_{f}(x_k,k), F_{h}(x_k,k), F_j(x_k,k) \}
\end{equation} where $F_f$ is the closed loop system under the foot stand policy $\pi_f(x_k)$, $F_h$ is the closed loop system subject to the hand stand policy $\pi_h$, and $\pi_j(x_k)$ is the jump forward $F_j(x_k,k)$ mode is the closed-loop system driven by a model predictive controller. The loss function is the sum of the individual subgoals $\ell(x_t, u_t) = \sum_i \omega_i l_i$ defined in Table \ref{table:reward-functions}:

\begin{table}[h!]
    \centering
    \setlength{\tabcolsep}{4pt}  
    \small
    \caption{Loss Terms}
    \begin{tabular}{c | c}
    \toprule
         \textbf{Loss Term} & \textbf{Expression} \\
         \midrule
        Height Tracking  & $l_\text{h} = 1.5 \cdot \|z^{tar} - z\|^2$  \\
        Orientation Tracking & $l_\text{ori} = 1.0 \cdot \| \phi_\text{body,xy} - \phi^\text{tar}_\text{body,xy}\|^2$\\
        Joint Position & $l_q = 0.5 \cdot \|q - q_\text{nominal}\|^2$  \\
        Action Rate & $l_\text{rate} = 0.001 \cdot \|u_t - u_{t-1}\|^2$ \\
        Energy Consumption & $l_\text{energy} = 0.003 \cdot \|\dot{q} \cdot \tau\|$ \\
        Pose Deviation & $l_\text{pose} = 2.5e-7 \cdot \exp\left(-\|q - q_\text{default}\|^2\right)$  \\
    \bottomrule
    \end{tabular}
    \vskip -5 pt
    \label{table:reward-functions}
\end{table}

For policy training process, we use PPO \cite{schulman2017proximalpolicyoptimizationalgorithms} as training backbone and adopt an asymmetric actor-critic settings \cite{pinto2017asymmetricactorcriticimagebased}, where the policy network and value network accept different observation inputs. Specifically, the policy network received the inputs mentioned earlier while the value network receive additional robot actuator forces and torso height. To effectively deploy the learned policy and model predictive control in simulation, we use Brax \cite{brax2021github} since it can effectively combine policy inference and online sampling in the same control framework without converting data type. \textbf{Policy and Value network design}: The objective of the policy to maximize the expected return of a policy. Here, we flip the sign of the cost function to represent reward function for learning. Both networks use a three-layer Multilayer Perceptron (MLP) with hidden sizes of 512, 256, and 128. In between each hidden layers, we use Swish\cite{ramachandran2017searchingactivationfunctions} as activation function. \textbf{MPC Design}: We adopt MPPI as backbone for sample-based model predictive control. The goal of the MPPI is to find best control perturbations of a future sequence of actions that best minimize the loss. We set planning horizon to be 20 steps forward and sample size to be 25. The temperature coefficient of MPPI is fixed to 0.1. 

\begin{figure}
  \begin{minipage}[t!]{.5\columnwidth}
    \includegraphics[width=\columnwidth]{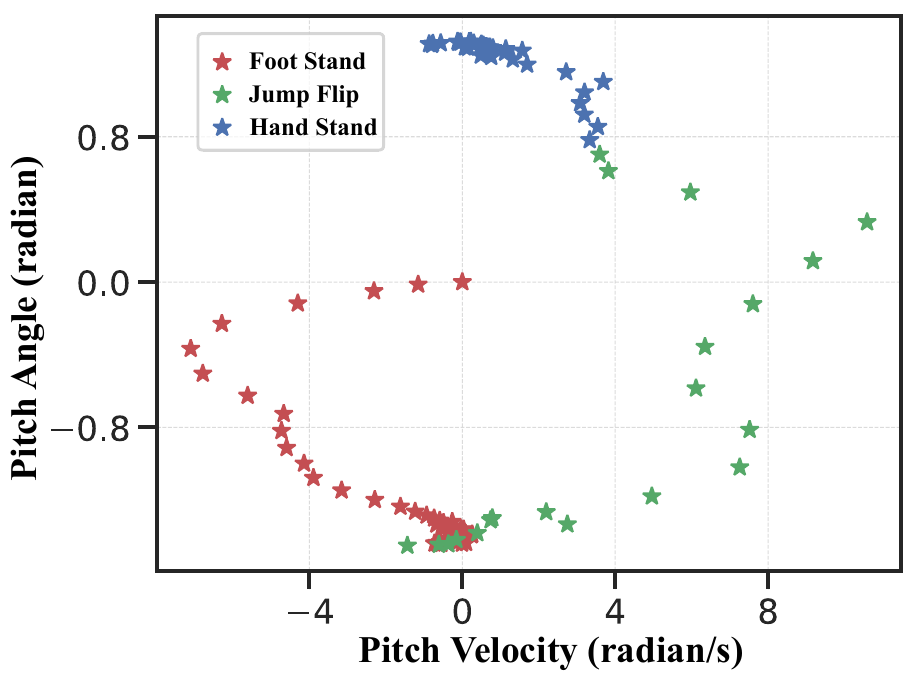}
  \end{minipage}
  \begin{minipage}[c]{.48\columnwidth}
    \vspace*{.05cm}
    \caption{\textbf{State Visitation of Robot Pitch Rotation:} We collect the trajectory robot motion and visualize the distribution of robot pitch angular position and velocity in the base frame. We show that our method can effective solving for control solutions that resides in distinct areas by optimizing control mode sequences.}
    \label{fig:state_visitation}
  \end{minipage}
  \vspace*{-.20cm}
\end{figure}

\subsection{Comparison to baseline control methods.}
We compare our approach against four baselines: (1) a unified policy network trained with PPO, referred to as PPO-Only; (2) pure model predictive control methods, including MPPI-Only~\cite{Mppi}, CEM-Only~\cite{cem}, and PS-Only (predictive sampling)~\cite{predictivesampling}; (3) predefined hybrid mode sequences, denoted as Predefined Mode Seq; and (4) our sample-based hybrid mode control combined with MPPI, denoted as Hybrid Mode Control. We choose these baselines because alternative hybrid-control and hierarchical reinforcement-learning approaches, such as the Logical Options Framework (LOF) \cite{pmlr-v139-araki21a}, are computationally expensive and require substantially more data, whereas our method composes complex locomotion behaviors from a simple policy-learning framework.
For evaluation, we assess whether each method can solve the task fully or partially, as well as the cumulative cost, which each method aims to minimize.
As shown in Table~\ref{tab:comparison}, we conducted five trials for each methods and our methods can sccuessfully perform complete flip and balance
\begin{table}
    \centering
    \setlength{\tabcolsep}{4pt}  
    \small
    \caption{Comparison with Existing Methods}
    \begin{tabular}{c | c c c | c}
    \toprule
         & \textbf{FootStand} & \textbf{JumpFlip} &  \textbf{HandStand} & \textbf{Cost} $\downarrow$ \\
         \midrule
        PPO-only  & \checkmark & $\times$ & $\times$ & 34.03 \\
        MPPI-Only & \checkmark & $\times$ & $\times$ & 46.73 \\
        CEM-Only & \checkmark & $\times$ & $\times$ & 55.68 \\
        PS-Only & \checkmark & $\times$ & $\times$ & 48.39 \\
        Fixed Seq & \checkmark & \checkmark & $\times$ & 22.24\\
        \textbf{Ours} & \checkmark & \checkmark & \checkmark & \textbf{13.519} \\
    \bottomrule
    \end{tabular}
    \vskip -10 pt
    \label{tab:comparison}
\end{table}

We observe that the robot falls quickly during the first transition period suggesting that solely using one control method is not enough to execute all motions. While training a single policy can successfully perform foot stand, the policy quickly fall during the flip phase. We hypothesis that a single MLP policy network cannot capture this multi-modal behaviors. 
We also implement a hybrid mode control method similar to ours, but with predefined modes.
The robot can safely transition to the flip phase but fails to adjust its pose into a handstand. In contrast, our approach is able to reason about the multi-modal nature of the objective and sequence the appropriate control modes. As shown in Fig.~\ref{fig:state_visitation}, each single-strategy controller learns to operate effectively only within a limited region of the state space. However, our method actively considers all single-strategy controllers and provides an effective way to generate more complex behaviors.

\subsection{Hardware Experiments}
We further validate our method in the real-world experiment using the quadrupedal robot Unitree Go2. As shown in Fig. \ref{fig:hardware_exps}, we successfully deploy our hybrid mode control method for agile foot stand, jump flip, and hand stand. The policy is converted to CPU runtime using Open Neural Network Exchange (ONNX) \cite{onnxruntime} to ensure real-time performance. We run the proposed method at 50 Hz on a single Intel i7-12700H CPU with 32 GB of memory. Additionally, we implement state estimation based on an extended Kalman filter using onboard sensing. In comparison to existing sample-based control framework \cite{fullordersampling, realtimewholebodycontrollegged} relying on high accurate mocap system, our method only uses onboard sensing, demonstrating robustness under noisy state measurements.

\section{Conclusions} 
\label{sec:conclusion}
This work presents a sample-based hybrid mode control method inspired by hybrid control theory. Our approach offers an alternative modeling scheme for sampling within a predictive-control framework. By actively searching for control modes, application time, and duration as integer-based optimization, our method can effectively be applied to tasks that require complex control compositions. Moreover, by reasoning over algorithmic and learned control modes, our method can synthesize complex behavior for high-dimensional systems.

There are several limitations and future directions to consider. Sample-based control typically requires an accurate contact model—one can only simulate what is well-represented in that model. This reliance on a precise model constrains real-world applications, particularly in unstructured environments or scenarios where obtaining a reliable model is challenging.
Future research directions include integrating our methods with data-driven approaches that do not require explicit modeling. 

\section{Acknowledgments}
This work is supported by the National Science Foundation
under award NSF FRR 2238066. Any opinions, findings, and
conclusions or recommendations expressed in this material are
those of the authors and do not necessarily reflect the views
of the National Science Foundation.

\bibliographystyle{IEEEtran} 
\bibliography{references}

@misc{You_1,
      title={Accelerating Visual-Policy Learning through Parallel Differentiable Simulation}, 
      author={Haoxiang You and Yilang Liu and Ian Abraham},
      year={2025},
      eprint={2505.10646},
      archivePrefix={arXiv},
      primaryClass={cs.LG},
      url={https://arxiv.org/abs/2505.10646},
}

@misc{You_2,
      title={Is Bellman Equation Enough for Learning Control?}, 
      author={Haoxiang You and Lekan Molu and Ian Abraham},
      year={2025},
      eprint={2503.02171},
      archivePrefix={arXiv},
      primaryClass={cs.LG},
      url={https://arxiv.org/abs/2503.02171},
}

@inproceedings{
  christian_2,
  title={Diversifying Parallel Ergodic Search: A Signature Kernel Evolution Strategy},
  author={Sreevardhan Sirigiri and Christian Hughes and Ian Abraham and Fabio Ramos},
  booktitle={The Thirty-ninth Annual Conference on Neural Information Processing Systems},
  year={2025},
  url={https://openreview.net/forum?id=3XuUnUEI7e},
}

@INPROCEEDINGS{wang_1,
  author={Liang, Chen and Wang, Qian and Xu, Andy and Rakita, Daniel},
  booktitle={2025 IEEE/RSJ International Conference on Intelligent Robots and Systems (IROS)}, 
  title={ad-trait: A Fast and Flexible Automatic Differentiation Library in Rust}, 
  year={2025},
  volume={},
  number={},
  pages={1320-1326},
  doi={10.1109/IROS60139.2025.11245908},
  ISSN={2153-0866},
  month={Oct},}

@misc{wang_2,
      title={Coherence-based Approximate Derivatives via Web of Affine Spaces Optimization}, 
      author={Daniel Rakita and Chen Liang and Qian Wang},
      year={2025},
      eprint={2504.18790},
      archivePrefix={arXiv},
      primaryClass={cs.RO},
      url={https://arxiv.org/abs/2504.18790},
}

@misc{wang_3,
      title={Subsecond 3D Mesh Generation for Robot Manipulation}, 
      author={Qian Wang and Omar Abdellall and Tony Gao and Xiatao Sun and Daniel Rakita},
      year={2025},
      eprint={2512.24428},
      archivePrefix={arXiv},
      primaryClass={cs.RO},
      url={https://arxiv.org/abs/2512.24428},
}

@article{liu_1,
    author = {Liu, Yilang and Barati Farimani, Amir},
    title = {An Energy-Saving Snake Locomotion Pattern Learned in a Physically Constrained Environment With Online Model-Based Policy Gradient Method},
    journal = {Journal of Mechanisms and Robotics},
    volume = {15},
    number = {4},
    pages = {041007},
    year = {2022},
    month = {11},
    issn = {1942-4302},
    doi = {10.1115/1.4055167},
    url = {https://doi.org/10.1115/1.4055167},
    eprint = {https://asmedigitalcollection.asme.org/mechanismsrobotics/article-pdf/15/4/041007/6941787/jmr_15_4_041007.pdf},
}

@misc{you2025acceleratingvisualpolicylearningparallel,
      title={Accelerating Visual-Policy Learning through Parallel Differentiable Simulation}, 
      author={Haoxiang You and Yilang Liu and Ian Abraham},
      year={2025},
      eprint={2505.10646},
      archivePrefix={arXiv},
      primaryClass={cs.LG},
      url={https://arxiv.org/abs/2505.10646}, 
}

@inproceedings{learningwalk,
  title={Learning to walk in minutes using massively parallel deep reinforcement learning},
  author={Rudin, Nikita and Hoeller, David and Reist, Philipp and Hutter, Marco},
  booktitle={Conference on robot learning},
  pages={91--100},
  year={2022},
  organization={PMLR}
}

@misc{Mppi,
      title={Information Theoretic Model Predictive Control: Theory and Applications to Autonomous Driving}, 
      author={Grady Williams and Paul Drews and Brian Goldfain and James M. Rehg and Evangelos A. Theodorou},
      year={2017},
      eprint={1707.02342},
      archivePrefix={arXiv},
      primaryClass={cs.RO},
      url={https://arxiv.org/abs/1707.02342}, 
}

@misc{Mppi2,
      title={Model Predictive Path Integral Control using Covariance Variable Importance Sampling}, 
      author={Grady Williams and Andrew Aldrich and Evangelos Theodorou},
      year={2015},
      eprint={1509.01149},
      archivePrefix={arXiv},
      primaryClass={cs.SY},
      url={https://arxiv.org/abs/1509.01149}, 
}

@misc{wholebodympc,
      title={Highly Dynamic Quadruped Locomotion via Whole-Body Impulse Control and Model Predictive Control}, 
      author={Donghyun Kim and Jared Di Carlo and Benjamin Katz and Gerardo Bledt and Sangbae Kim},
      year={2019},
      eprint={1909.06586},
      archivePrefix={arXiv},
      primaryClass={cs.RO},
      url={https://arxiv.org/abs/1909.06586}, 
}

@misc{predictivesampling,
      title={Predictive Sampling: Real-time Behaviour Synthesis with MuJoCo}, 
      author={Taylor Howell and Nimrod Gileadi and Saran Tunyasuvunakool and Kevin Zakka and Tom Erez and Yuval Tassa},
      year={2022},
      eprint={2212.00541},
      archivePrefix={arXiv},
      primaryClass={cs.RO},
      url={https://arxiv.org/abs/2212.00541}, 
}

@misc{brax2021github,
  author = {C. Daniel Freeman and Erik Frey and Anton Raichuk and Sertan Girgin and Igor Mordatch and Olivier Bachem},
  title = {Brax - A Differentiable Physics Engine for Large Scale Rigid Body Simulation},
  url = {http://github.com/google/brax},
  version = {0.11.0},
  year = {2021},
}

@article{pinto2017asymmetricactorcriticimagebased,
  title={Asymmetric Actor Critic for Image-Based Robot Learning},
  author={Pinto, Lerrel and Andrychowicz, Marcin and Welinder, Peter and Zaremba, Wojciech and Abbeel, Pieter},
  journal={Robotics: Science and Systems XIV},
  year={2018},
  publisher={Robotics: Science and Systems Foundation}
}

@misc{ramachandran2017searchingactivationfunctions,
      title={Searching for Activation Functions}, 
      author={Prajit Ramachandran and Barret Zoph and Quoc V. Le},
      year={2017},
      eprint={1710.05941},
      archivePrefix={arXiv},
      primaryClass={cs.NE},
      url={https://arxiv.org/abs/1710.05941}, 
}

@misc{schulman2017proximalpolicyoptimizationalgorithms,
      title={Proximal Policy Optimization Algorithms}, 
      author={John Schulman and Filip Wolski and Prafulla Dhariwal and Alec Radford and Oleg Klimov},
      year={2017},
      eprint={1707.06347},
      archivePrefix={arXiv},
      primaryClass={cs.LG},
      url={https://arxiv.org/abs/1707.06347}, 
}

@INPROCEEDINGS{6094601,

  author={Chipalkatty, Rahul and Daepp, Hannes and Egerstedt, Magnus and Book, Wayne},

  booktitle={2011 IEEE/RSJ International Conference on Intelligent Robots and Systems}, 

  title={Human-in-the-loop: MPC for shared control of a quadruped rescue robot}, 

  year={2011},

  volume={},

  number={},

  pages={4556-4561},

  doi={10.1109/IROS.2011.6094601}}

@INPROCEEDINGS{6426270,

  author={Caldwell, T. M. and Murphey, T. D.},

  booktitle={2012 IEEE 51st IEEE Conference on Decision and Control (CDC)}, 

  title={Projection-based switched system optimization: Absolute continuity of the line search}, 

  year={2012},

  volume={},

  number={},

  pages={699-706},

  doi={10.1109/CDC.2012.6426270}}

@INPROCEEDINGS{6669577,

  author={Flaßkamp, Kathrin and Murphey, Todd and Ober-Blöbaum, Sina},

  booktitle={2013 European Control Conference (ECC)}, 

  title={Discretized switching time optimization problems}, 

  year={2013},

  volume={},

  number={},

  pages={3179-3184},

  doi={10.23919/ECC.2013.6669577}}

@article{EGERSTEDT1999617,
title = {A hybrid control approach to action coordination for mobile robots},
journal = {IFAC Proceedings Volumes},
volume = {32},
number = {2},
pages = {617-622},
year = {1999},
note = {14th IFAC World Congress 1999, Beijing, Chia, 5-9 July},
issn = {1474-6670},
doi = {https://doi.org/10.1016/S1474-6670(17)56105-4},
url = {https://www.sciencedirect.com/science/article/pii/S1474667017561054},
author = {M. Egerstedt and X. Hu and A. Stotsky}
}

@INPROCEEDINGS{1272934,

  author={Egerstedt, M. and Wardi, Y. and Delmotte, F.},

  booktitle={42nd IEEE International Conference on Decision and Control (IEEE Cat. No.03CH37475)}, 

  title={Optimal control of switching times in switched dynamical systems}, 

  year={2003},

  volume={3},

  number={},

  pages={2138-2143 Vol.3},

  doi={10.1109/CDC.2003.1272934}}

@misc{onnxruntime,
  title={ONNX Runtime},
  author={ONNX Runtime developers},
  year={2021},
  howpublished={\url{https://onnxruntime.ai/}},
  note={Version: x.y.z}
}

@inproceedings{realtimewholebodycontrollegged,
  title={Real-time whole-body control of legged robots with model-predictive path integral control},
  author={Alvarez-Padilla, Juan and Zhang, John Z and Kwok, Sofia and Dolan, John M and Manchester, Zachary},
  booktitle={2025 IEEE International Conference on Robotics and Automation (ICRA)},
  pages={14721--14727},
  year={2025},
  organization={IEEE}
}

@inproceedings{liu2020design,
  title={Design an Augmentation Exoskeleton to Enhance Lifting Strength},
  author={Liu, Yilang and Chen, Yuhao and Li, Hongyi and Dong, Janet},
  booktitle={ASME International Mechanical Engineering Congress and Exposition},
  volume={84539},
  pages={V006T06A022},
  year={2020},
  organization={American Society of Mechanical Engineers}
}

@misc{fullordersampling,
      title={Full-Order Sampling-Based MPC for Torque-Level Locomotion Control via Diffusion-Style Annealing}, 
      author={Haoru Xue and Chaoyi Pan and Zeji Yi and Guannan Qu and Guanya Shi},
      year={2024},
      eprint={2409.15610},
      archivePrefix={arXiv},
      primaryClass={cs.RO},
      url={https://arxiv.org/abs/2409.15610}, 
}

@article{hybridcontrol2,
  title={Model-Based Generalization Under Parameter Uncertainty Using Path Integral Control},
  author={Ian Abraham and Ankur Handa and Nathan D. Ratliff and Kendall Lowrey and Todd D. Murphey and Dieter Fox},
  journal={IEEE Robotics and Automation Letters},
  year={2020},
  volume={5},
  pages={2864-2871},
  url={https://api.semanticscholar.org/CorpusID:212635299}
}

@INPROCEEDINGS{hybridctrl_bipedal,
  author={Jong Hyeon Park and Hoam Chung},
  booktitle={Proceedings 1999 IEEE International Conference on Robotics and Automation (Cat. No.99CH36288C)}, 
  title={Hybrid control for biped robots using impedance control and computed-torque control}, 
  year={1999},
  volume={2},
  number={},
  pages={1365-1370 vol.2},
  doi={10.1109/ROBOT.1999.772551}}

@article{hybridctrl_bipedal2,
  title={Human-inspired control of bipedal walking robots},
  author={Ames, Aaron D},
  journal={IEEE Transactions on Automatic Control},
  volume={59},
  number={5},
  pages={1115--1130},
  year={2014},
  publisher={IEEE}
}

@INPROCEEDINGS{swithtimeop,
  author={Katayama, Sotaro and Ohtsuka, Toshiyuki},
  booktitle={2022 IEEE/RSJ International Conference on Intelligent Robots and Systems (IROS)}, 
  title={Whole-body model predictive control with rigid contacts via online switching time optimization}, 
  year={2022},
  volume={},
  number={},
  pages={8858-8865},
  doi={10.1109/IROS47612.2022.9981790}}

@ARTICLE{QuadrupedCapturability,
  author={Chen, Hua and Hong, Zejun and Yang, Shunpeng and Wensing, Patrick M. and Zhang, Wei},
  journal={IEEE Transactions on Robotics}, 
  title={Quadruped Capturability and Push Recovery via a Switched-Systems Characterization of Dynamic Balance}, 
  year={2023},
  volume={39},
  number={3},
  pages={2111-2130},
  doi={10.1109/TRO.2023.3240622}}

@inproceedings{swithtimeop3,
  title={Hybrid learning-and model-based planning and control of in-hand manipulation},
  author={Zarrin, Rana Soltani and Jitosho, Rianna and Yamane, Katsu},
  booktitle={2023 IEEE/RSJ International Conference on Intelligent Robots and Systems (IROS)},
  pages={8720--8726},
  year={2023},
  organization={IEEE}
}

@ARTICLE{hybridctrl_1,
  author={Kong, Nathan J. and Li, Chuanzheng and Council, George and Johnson, Aaron M.},
  journal={IEEE Transactions on Robotics}, 
  title={Hybrid iLQR Model Predictive Control for Contact Implicit Stabilization on Legged Robots}, 
  year={2023},
  volume={39},
  number={6},
  pages={4712-4727},
  doi={10.1109/TRO.2023.3308773}}

@article{seqactionctrl,
   title={Sequential Action Control: Closed-Form Optimal Control for Nonlinear and Nonsmooth Systems},
   volume={32},
   ISSN={1941-0468},
   url={http://dx.doi.org/10.1109/TRO.2016.2596768},
   DOI={10.1109/tro.2016.2596768},
   number={5},
   journal={IEEE Transactions on Robotics},
   publisher={Institute of Electrical and Electronics Engineers (IEEE)},
   author={Ansari, Alexander R. and Murphey, Todd D.},
   year={2016},
   month=oct, pages={1196–1214} }

@ARTICLE{hybridctrl4,
  author={Carpentier, Justin and Mansard, Nicolas},
  journal={IEEE Transactions on Robotics}, 
  title={Multicontact Locomotion of Legged Robots}, 
  year={2018},
  volume={34},
  number={6},
  pages={1441-1460},
  doi={10.1109/TRO.2018.2862902}}

@article{cem,
author = {Mannor, Shie and Rubinstein, Reuven and Gat, Yohai},
year = {2003},
month = {07},
pages = {},
title = {The Cross Entropy method for Fast Policy Search},
volume = {2},
journal = {Proceedings, Twentieth International Conference on Machine Learning}
}

@INPROCEEDINGS{es,
  author={Wierstra, Daan and Schaul, Tom and Peters, Jan and Schmidhuber, Juergen},
  booktitle={2008 IEEE Congress on Evolutionary Computation (IEEE World Congress on Computational Intelligence)}, 
  title={Natural Evolution Strategies}, 
  year={2008},
  volume={},
  number={},
  pages={3381-3387},
  doi={10.1109/CEC.2008.4631255}}

@article{simp_model2,
author = {Wensing, Patrick and Orin, David},
year = {2015},
month = {09},
pages = {1550039},
title = {Improved Computation of the Humanoid Centroidal Dynamics and Application for Whole-Body Control},
volume = {13},
journal = {International Journal of Humanoid Robotics},
doi = {10.1142/S0219843615500395}
}

@ARTICLE{simp_model5,
  author={Fernbach, Pierre and Tonneau, Steve and Stasse, Olivier and Carpentier, Justin and Taïx, Michel},
  journal={IEEE Transactions on Robotics}, 
  title={C-CROC: Continuous and Convex Resolution of Centroidal Dynamic Trajectories for Legged Robots in Multicontact Scenarios}, 
  year={2020},
  volume={36},
  number={3},
  pages={676-691},
  doi={10.1109/TRO.2020.2964787}}

@incollection{liu2010sampling,
  title={Sampling-based contact-rich motion control},
  publisher={Association for Computing Machinery},
  author={Liu, Libin and Yin, KangKang and Van de Panne, Michiel and Shao, Tianjia and Xu, Weiwei},
  booktitle={ACM SIGGRAPH 2010 papers},
  pages={1--10},
  year={2010}
}

@ARTICLE{sample_manipulation1,
  author={Pang, Tao and Suh, H. J. Terry and Yang, Lujie and Tedrake, Russ},
  journal={IEEE Transactions on Robotics}, 
  title={Global Planning for Contact-Rich Manipulation via Local Smoothing of Quasi-Dynamic Contact Models}, 
  year={2023},
  volume={39},
  number={6},
  pages={4691-4711},
  doi={10.1109/TRO.2023.3300230}}

@ARTICLE{sample_manipulation2,
  author={Rizzi, Giuseppe and Chung, Jen Jen and Gawel, Abel and Ott, Lionel and Tognon, Marco and Siegwart, Roland},
  journal={IEEE Transactions on Robotics}, 
  title={Robust Sampling-Based Control of Mobile Manipulators for Interaction With Articulated Objects}, 
  year={2023},
  volume={39},
  number={3},
  pages={1929-1946},
  doi={10.1109/TRO.2022.3233343}}

@misc{sample_manipulation3,
      title={STORM: An Integrated Framework for Fast Joint-Space Model-Predictive Control for Reactive Manipulation}, 
      author={Mohak Bhardwaj and Balakumar Sundaralingam and Arsalan Mousavian and Nathan Ratliff and Dieter Fox and Fabio Ramos and Byron Boots},
      year={2021},
      eprint={2104.13542},
      archivePrefix={arXiv},
      primaryClass={cs.RO},
      url={https://arxiv.org/abs/2104.13542}, 
}

@misc{ansari_sequential_nodate,
    title = {Sequential Action Control for Tracking of Free Invariant Manifolds},
    journal = {IFAC-PapersOnLine},
    volume = {48},
    number = {27},
    pages = {335-342},
    year = {2015},
    note = {Analysis and Design of Hybrid Systems ADHS},
    issn = {2405-8963},
    doi = {https://doi.org/10.1016/j.ifacol.2015.11.197},
    url = {https://www.sciencedirect.com/science/article/pii/S2405896315024556},
    author = {Alex Ansari and Kathrin Flaßkamp and Todd D. Murphey},
}

@InProceedings{pmlr-v139-araki21a,
  title = 	 {The Logical Options Framework},
  author =       {Araki, Brandon and Li, Xiao and Vodrahalli, Kiran and Decastro, Jonathan and Fry, Micah and Rus, Daniela},
  booktitle = 	 {Proceedings of the 38th International Conference on Machine Learning},
  pages = 	 {307--317},
  year = 	 {2021},
  editor = 	 {Meila, Marina and Zhang, Tong},
  volume = 	 {139},
  series = 	 {Proceedings of Machine Learning Research},
  month = 	 {18--24 Jul},
  publisher =    {PMLR},
  url = 	 {https://proceedings.mlr.press/v139/araki21a.html}
}

\end{document}